\documentclass[conference]{IEEEtran}

\usepackage{amsmath,amssymb,amsthm}
\usepackage{booktabs}
\usepackage{tabularx}
\usepackage{enumitem}
\usepackage{xcolor}
\usepackage{microtype}
\usepackage{array}
\usepackage{graphicx}
\usepackage{cite}
\usepackage[hyphens]{url}
\usepackage[hidelinks]{hyperref}

\setlist[itemize]{leftmargin=1.2em,topsep=3pt,itemsep=2.5pt,parsep=0pt}
\setlist[enumerate]{leftmargin=1.5em,topsep=3pt,itemsep=2.5pt,parsep=0pt}

\newcolumntype{Y}{>{\raggedright\arraybackslash}X}

\theoremstyle{definition}
\newtheorem{definition}{Definition}
\newtheorem{assumption}{Assumption}
\theoremstyle{plain}
\newtheorem{theorem}{Theorem}
\newtheorem{proposition}{Proposition}
\newtheorem{corollary}{Corollary}
\theoremstyle{remark}
\newtheorem{remark}{Remark}

\title{Are You Still the Agent I Authorized?\\
Earned Authority under a Fixed Ceiling for Evolving Agents}

\author{
\IEEEauthorblockN{Zhaoxi Zhang}
\IEEEauthorblockA{University of Technology Sydney\\
zhaoxi.zhang-1@student.uts.edu.au}
\and
\IEEEauthorblockN{Xiaomei Zhang}
\IEEEauthorblockA{Griffith University\\
xiaomei.zhang@griffithuni.edu.au}
}

\begin{document}
\bstctlcite{IEEEtran:BSTcontrol}
\maketitle

\begin{abstract}
Long-lived AI agents increasingly evolve after deployment by retaining experience, acquiring skills and tools, revising workflows, delegating work, and moving across task phases.
This improves adaptation but creates a distinct authorization problem.
Tool-enabled agents can turn model errors and prompt injections into consequential external actions; when evolution occurs under a live grant, the subject exercising that authority or the context in which it acts may no longer match what the user evaluated.
Evolution can change both the effects reachable under an old grant and the authority required by the task, which may rise, fall, or become incomparable.
Existing tool policies constrain actions but do not determine when a grant survives this change.

We formulate \emph{authorization continuity}: when does an existing grant remain valid, how may active authority change, and what boundary must never move?
Our state-bound model fixes a transition envelope and an immutable effect ceiling at grant time.
The envelope determines whether the grant survives a mutation; below the ceiling, authority may contract freely and expand only under specified evidence conditions.
We distinguish requested from realized effects and prove that, under complete mediation, sound effect abstraction, attenuating delegation, and monitor integrity, mutation cannot amplify protected effects beyond the user-issued ceiling.
Agent-produced evidence may allocate authority below the ceiling but cannot raise it.
Finally, we map six mutation classes to their authorization consequences.
\end{abstract}

\section{Introduction}
\label{sec:introduction}

Agents deployed for open-ended, long-lived work cannot be configured in advance for every task, tool, failure mode, or change in their operating context.
Knowledge becomes outdated, new services appear, users develop local conventions, and long-horizon tasks expose errors that were not represented during training.
This limitation motivates lifelong and self-evolving agents that adapt after deployment by retaining feedback, extracting reusable experience, accumulating skills, or revising their own control logic \cite{gao2026selfevolvingsurvey, zheng2025lifelong}.
Reflexion stores linguistic feedback in episodic memory, ExpeL distills experience across tasks, and Voyager accumulates executable skills through interaction \cite{shinn2023reflexion, zhao2023expel, wang2023voyager}.
Automated Design of Agentic Systems, G\"odel Agent, and the Darwin G\"odel Machine extend adaptation to prompts, workflows, agent code, or the improvement process itself \cite{hu2024adas, yin2024godelagent, zhang2025dgm}.
Across these mechanisms, evolution is a practical response to a common bottleneck: a useful long-lived agent must learn from interaction rather than repeatedly start from its deployment-time state.

\begin{figure}[!t]
\centering
\includegraphics[width=\linewidth]{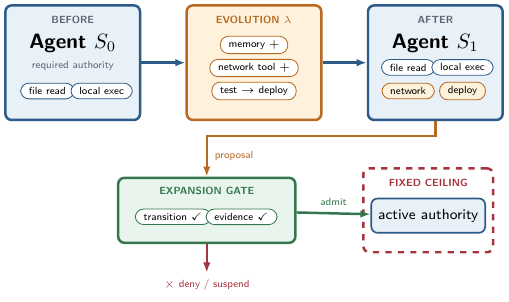}
\caption{
    An example mutation stores new memory, adds a network tool, and enters deployment, expanding the authority needed by the task.
    The changed need produces only a proposal: expansion requires both transition compatibility and admissible evidence.
    Active authority remains inside the ceiling fixed at grant time; failure of either guard denies the expansion or suspends the grant.
}
\label{fig:teaser}
\end{figure}

These gains arrive against an already fragile security baseline: deployed LLMs remain vulnerable to adversarial attacks \cite{zhang2022evaluating, zhang2022self, zhang2024stealing, zhang2025exploring, zhang2026character, wang2025privacy, fan2026does, wang2024cryptgraph, long2024fedcl, zhang2026less, wang2022towards, wang2024towards, shu2026trustworthy}, and tool-enabled agents raise the stakes because model-generated decisions can read sensitive data, modify external state, and trigger consequential actions \cite{beurerkellner2025patterns, debenedetti2025camel, nist2026redteam}.
Evolution adds a further complication: it may change the subject exercising an existing grant or the authorization context in which that grant is used.
A user authorizes an agent under particular memory, tools, workflow, task, expected effects, and enforcement conditions.
Suppose a coding agent authorized to diagnose a test failure later stores untrusted tool output as a reusable rule, replaces a local test routine with a network-backed skill, or delegates the next phase to a release subagent.
If the session grant persists, the evolved state can inherit repository, network, credential, or deployment authority even though the user evaluated only its predecessor.
The security problem therefore appears as soon as mutation and authorization overlap in time: when does an old grant remain valid, how may active authority change, and what boundary must no evolution cross?

We use \emph{evolution} for this broad phenomenon and \emph{mutation} for one concrete state transition within it.
The relevant state includes persistent memory, instructions, skills, tools, workflow code, delegation structure, task phase, and enforcement configuration, not only model weights.
Some mutations occur across generations; others occur while one user grant and one session credential remain active.
Evolution also converts transient inputs into durable control state: repository content, tool output, or web data can be distilled into memory, a skill, or a modified workflow that affects later protected actions.
Recent studies find that even benign experience can reduce safety in high-risk settings, that locally correct experiences can be poisoned into harmful general rules, and that malicious skills or memory entries can create persistent influence \cite{zhao2026experience, wang2026oep, chen2026dymalskill, dash2026memory}.
These results establish risks to evolving behavior; our concern is the additional amplification that occurs when the changed behavior inherits authority granted to an earlier state.

Authorization can drift along two distinct dimensions.
First, evolution can expand the effects reachable under an unchanged grant; capability-binding work independently motivates revalidation after tool-manifest changes \cite{zhou2026dynamiccapabilities}.
Second, evolution can change the authority actually needed to complete the task.
The need can increase when a new plan adds a remote service, decrease when a learned local routine replaces one, or become incomparable when alternative valid plans require different effect types.
Most sharply, it can change without any agent component or tool manifest changing at all: moving from editing and testing to pushing or deployment introduces a more consequential effect solely because the task entered a new phase.
The directional claims and the phase-only case are our authorization inference and are formalized by Proposition~\ref{prop:needmoves}; we do not attribute them to the evolving-agent literature.

This creates a three-way tension that static authorization cannot resolve.
A broad session grant preserves task completion but conflicts with least privilege and silently transfers excess authority to later versions of the principal \cite{saltzer1975protection, li2026conleash}.
A narrow immutable grant preserves the boundary but eventually prevents a useful evolving agent from completing work, or recreates repeated confirmation and its associated consent fatigue \cite{li2026conleash}.
Existing privilege-control work correspondingly treats automatic narrowing differently from policy expansion, which requires a stronger approval path \cite{shi2025progent}.
Our additional threat-model inference is that automatically tracking the agent's claimed need is not such a path: it lets mutable agent state participate in justifying its own privilege expansion (Section~\ref{sec:minimal-gate} and Appendix~\ref{sec:reach}).
Thus neither a static grant nor an adjustment controlled by the evolving agent is sufficient.
The design goal is conditional authority change beneath a boundary fixed outside the evolution process.

Neither higher competence nor successful past behavior resolves this problem.
Competence evidence says whether an agent is likely to use authority effectively; it does not establish that the agent resists injection, that a new tool preserves an old effect boundary, or that a descendant should inherit the parent's trust.
Moreover, task outcomes, tests, error rates, and execution histories are often produced by the same session that benefits from an authority increase.
They may be useful reliability signals, but treating them as the only security boundary gives the beneficiary partial control over its own privilege.

Existing authorization work provides important pieces of the solution.
Tool policies constrain calls and arguments, typed or capability-based systems restrict effect flow, skill policies bind permissions to reusable artifacts, and capability binding invalidates grants when a manifest changes \cite{shi2025progent, zhou2026lbac, pan2026skillguard, debenedetti2025camel, zhou2026dynamiccapabilities}.
These mechanisms primarily ask what a policy permits or whether a named capability remains present.
They do not by themselves answer whether an unchanged policy remains valid after memory, workflow, model, delegation, enforcement, or task-phase evolution.
The missing object is \emph{authorization continuity}: a rule for deciding when an old grant survives mutation of the authorized subject or context and how much authority may be active afterward.

We separate three questions that are often conflated:
\begin{enumerate}
    \item Did the agent or task change in a way relevant to this grant?
    \item If the grant remains valid, what authority should be active now?
    \item What authority can no runtime evidence ever exceed?
\end{enumerate}
The first is a continuity question.
The second is an allocation question.
The third is the security boundary.

Our answer separates two guards that respond to different failures.
A \emph{transition envelope} states which mutations the old grant may survive; crossing it suspends the grant even if the currently requested effect remains small.
An immutable \emph{effect ceiling} bounds all authority available under the grant; no observation generated during the session may raise it.
Within that ceiling, active authority may contract freely and expand only when specified evidence conditions hold.
This allows adaptation without forcing competence evidence to carry an adversarial security guarantee that it cannot support.
The separation yields the central design obligation:
\smallskip

\noindent\emph{A system may condition authority on demonstrated performance only if it also fixes a bound that no performance can move.}\par
\smallskip

\noindent\textbf{Contributions.}
\begin{itemize}
    \item We formulate \emph{authorization continuity} for evolving agents as a grant-relative transition-validity guard, distinct from effect containment.
    \item We define conditional authority beneath an immutable ceiling: descent is free, ascent is evidence-gated, and no runtime evidence can move the user-issued boundary.
    \item We give a non-amplification result that distinguishes requested from realized effects and states the closure-sound monitor-view obligation explicitly.
    \item We map evolution mechanisms to authorization consequences and systematize six mutation classes that make the problem empirically testable.
\end{itemize}

\section{Evolving Agents and Authorization Drift}
\label{sec:background}

\subsection{What Evolves}
\label{sec:evolution}

We model the authorized subject as a model plus its harness: instructions, planning loop, memory, tools, skills, hooks, and delegation interfaces.
This is the security principal whose evolution the Introduction motivates.
Treating only the underlying model as the principal would miss changes to the mechanisms that select, parameterize, and execute its actions.
We distinguish subject mutation from authorization-context mutation.
The former changes the model or harness; the latter changes the task phase, trust context, or enforcement conditions under which the subject acts.
Either can make an old grant stale, so the state monitored below contains both even though only the former changes the principal itself.

For authorization, the many evolution surfaces can be compressed into three distinct dimensions:
\begin{itemize}
    \item \textbf{Competence}: what the agent can infer or plan.
    \item \textbf{Affordance}: what actions the harness makes expressible, such as a file or deployment tool.
    \item \textbf{Authority}: what effects the surrounding system will actually permit on the user's behalf.
\end{itemize}
Evolution can change any of the three, but only authority determines which protected effects the runtime permits.
An increase in competence does not itself justify an increase in the security boundary.

\subsection{Scope and Non-Goals}
\label{sec:scope-main}

Our scope is delegated authority over externally observable effects while agent or task state changes.
Mutation may be benign, adversarial, or an unintended consequence of optimizing task success \cite{shao2026misevolution, wang2026oep}.
We do not solve model alignment, infer latent user intent, or require a new dialog after every prompt or model update.
We ask only whether an old grant remains usable and what effects it can bound before the next protected action.

\section{System and Threat Model}
\label{sec:threat}

\subsection{Mutable Agent State}
\label{sec:state}

Let $S_t$ denote the grant-relevant state at time $t$.
It combines the authorized subject (model identity, harness control state, persistent memory, capability configuration, and delegation) with its authorization context, including task phase and enforcement state.
We do not expose these components as separate symbols in the main text because the monitor need only classify their changes relative to a grant.
A session evolves through labeled transitions $S_t \xrightarrow{\lambda} S_{t+1}$.
A label names a security-relevant class, such as skill loading, tool addition, subagent creation, or entry into a deployment phase.

\subsection{Executions, Labels, and Effects}
\label{sec:executions}

For an action $a$, $\mathsf{Req}(a)$ is the monitor-verified set of protected effects requested at its boundary.
$\mathsf{Eff}(S,a)$ is the set actually produced when $a$ executes in state $S$.
The distinction matters because a changed tool implementation or harness may realize more than its nominal call suggests.
Effects include a type and the resources touched, for example a local file update on a repository path or a remote update to a named branch.

\subsection{Trusted Computing Base}
\label{sec:tcb}

An external monitor mediates protected effects, stores grants, and cannot be modified by the agent authority it enforces.
The model, instructions, memory, skills, generated code, tools, and subagents may be adversarial.
The monitor sees an abstraction $\hat{S}_t$, not the full state $S_t$.
The abstraction is sufficient only if the monitor-verified request remains a sound bound:
\begin{equation}
    \mathsf{Eff}(S_t,a) \subseteq \mathsf{Req}(a).
    \label{eq:request-sound}
\end{equation}
This is a semantic obligation on the monitor view and its enforcement, not an assumption that $\hat{S}_t=S_t$.
Unobserved mutation is one way to violate it.

\subsection{Attacker Model}
\label{sec:attacker}

The attacker may control ingested content, skills, remote tools, and memory, and may thereby steer the agent's behavior \cite{chen2026dymalskill, dash2026memory, wang2026oep, owasp2026skills}.
The attacker does not control the external monitor or user-issued ceilings.
We consider two successes: an effect outside all user-issued ceilings, or reuse of a grant after an incompatible transition even when effects remain inside the ceiling.
Side channels, monitor implementation defects, and information-flow confidentiality are outside scope.

\section{When Evolution Changes Authority}
\label{sec:evolution-authority}

The authority need of a task state is the family of protected-effect sets under which at least one completing plan exists.
It need not have a unique minimum because alternative plans may require incomparable resources or effect types.
Evolution changes this family when it adds or removes effectful steps, changes the available completing plans, or moves the task to a phase with different consequences.
Skills, workflows, memory, curricula, and agent compositions can all evolve in these ways \cite{wang2023voyager, hu2024adas, zhang2025dgm, gao2026selfevolvingsurvey}.
The authorization consequences below are our inference from those mechanisms rather than claims made by the cited systems.

Need may rise, fall, become incomparable, or remain unchanged.
Replacing a local test runner with one that reports to a remote service adds network transmission to every completing plan and raises a minimal requirement.
Replacing a networked check with a local analyzer can lower it.
Two alternative valid plans can instead require incomparable resources, such as access to a local secret or to a remote service.
Appendix Proposition~\ref{prop:needmoves} formalizes these directions and the phase-only case.

Most sharply, need can change without any subject mutation.
Moving from editing and testing to pushing a branch introduces a remote update while the model, harness, and capability manifest remain bit-identical.
Task phase is therefore part of the authorization context, not merely application metadata.
Conversely, a model upgrade may leave required and reachable effects unchanged while still crossing a transition envelope that binds the grant to a provider or trust class.
The appropriate response is consequently grant-relative: recompute active authority when required effects change, suspend when a continuity condition fails, and otherwise preserve the grant.

\section{A Minimal State-Bound Authorization Model}
\label{sec:minimal-model}

\subsection{Fixed Boundary, Conditional Authority}
\label{sec:minimal-bound}

A grant $G$ fixes three things at authorization time: an effect ceiling $B$, an initial active authority $P_0 \subseteq B$, and a transition envelope listing the mutation classes that the grant may survive.
The ceiling is a set of protected effects.
A structured implementation may encode that set with resource and effect-type fields while carrying budgets, goals, phases, and delegation constraints alongside it.

At time $t$, the runtime proposes $P_t^{\mathrm{prop}}$.
Contraction is always permitted.
An expansion or an incomparable change requires the grant's evidence condition.
In every case the monitor applies one clamp:
\begin{equation}
    P_t = P_t^{\mathrm{prop}} \cap B.
    \label{eq:clamp}
\end{equation}
Only a new user grant can replace $B$.
Evidence may therefore determine which part of a previously consented boundary is active, but cannot enlarge the boundary.

The transition envelope is independent of this order.
Let $\mathsf{Alive}_G(t)$ mean that every mutation label observed since grant $G$ was issued belongs to its envelope.
If it is false, the monitor suspends the grant before the next protected action, even when $P_t \subseteq B$ still holds.
This separates authorization continuity from effect containment.
Figure~\ref{fig:guards} makes the independence of the two decisions explicit.

\begin{figure}[t]
\centering
\includegraphics[width=\linewidth]{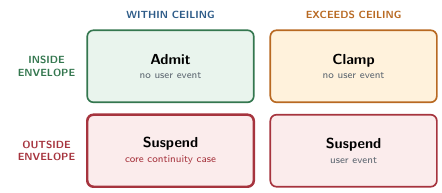}
\caption{
    The transition envelope decides whether the grant survives, while the ceiling determines how much authority may remain active.
    The highlighted cell is the paper's central continuity case: effects remain bounded, but the subject or context no longer matches the grant.
}
\label{fig:guards}
\end{figure}

\subsection{Admission and the Security Bound}
\label{sec:minimal-admission}

The monitor admits action $a_t$ only if the grant is alive, $\mathsf{Req}(a_t) \subseteq P_t$, its budget and delegation constraints hold, and any goal or phase condition succeeds.
The first three checks are mechanical.
Goal compatibility may require a conservative semantic oracle and can cause false rejection without weakening the effect bound.

\begin{theorem}[Non-amplification, compact form]
\label{thm:compact-nonamp}
Suppose protected effects are completely mediated, Equation~\ref{eq:request-sound} holds at every admission, each active or delegated authority is contained in a user-issued ceiling, and the monitor is outside the agent's control.
Then every protected effect in an execution belongs to at least one user-issued ceiling.
\end{theorem}

\begin{proof}
For an admitted action under grant $G$, request soundness and admission give
\begin{equation}
    \mathsf{Eff}(S_t,a_t) \subseteq \mathsf{Req}(a_t)
    \subseteq P_t \subseteq B.
\end{equation}
Complete mediation applies this chain to every protected effect.
The assumed containment invariant applies equally to delegated authority.
\end{proof}

This compact statement takes containment of active and delegated authority as an abstract invariant.
Appendix Theorem~\ref{thm:nonamp} gives its structured counterpart: the fixed clamp, finite delegation ancestry, attenuating-delegation assumption, and scope monotonicity together imply containment in a user-issued root ceiling, while closure soundness makes the monitor-view obligation explicit.
The detailed result therefore refines the compact argument rather than merely renaming its sets.

The theorem is deliberately a boundary result.
It does not say that the current principal is acceptable, that the action serves the user's intent, or that information cannot flow between separately permitted effects.
Those claims require the transition envelope, semantic checks, and information-flow control respectively.

\subsection{What Conditional Ascent Can Claim}
\label{sec:minimal-gate}

A test pass, low error rate, or successful task history may be useful evidence that more of $B$ should be active.
If the agent can produce the qualifying evidence itself, however, the security guarantee remains $B$.
The evidence rule then improves expected exposure or reliability, not the adversarial upper bound.
A tighter security bound requires at least one premise that the session cannot realize under its current authority, such as an external attestation, an independently controlled evaluator, or a fresh human countersignature.
Appendix~\ref{sec:reach} gives a history-independent least-closure construction for a sound gate-dependent permission bound.
It is an over-approximation and is not claimed to enumerate effects that some compatible execution must realize.

\section{Authorization-Relevant Mutations}
\label{sec:taxonomy}

The model still needs an operational answer to the first question: which changes should a monitor observe?
Table~\ref{tab:mutations} organizes six mutation classes by why an earlier grant may become stale.
They are not a partition.
A memory update that becomes a reusable instruction, for example, is both a trust-context change and a control-state change.
The classification is a sensitivity checklist: it identifies what must be re-evaluated, while the grant determines whether the response is continuation, contraction, suspension, or a new authorization request.
In Table~\ref{tab:mutations}, ``any direction'' abbreviates rise, fall, an incomparable change, or no change.

\begin{table*}[t]
\centering
\footnotesize
\caption{
    Grant-relevant subject and context mutations.
    A mutation does not automatically require reauthorization; the response is relative to the grant's transition envelope and effect ceiling.
}
\label{tab:mutations}
\begin{tabularx}{\textwidth}{p{0.10\textwidth} p{0.20\textwidth} Y p{0.13\textwidth} p{0.20\textwidth}}
\toprule
Class & Representative changes & Why the old grant may be stale & Possible need change & Typical response \\
\midrule
Control state & Model or provider switch; instruction, skill, hook, harness, or persistent-memory update & The policy selecting actions, interpreting intent, or handling failures has changed & Any direction & Re-measure the subject; continue only within an approved transition class \\
\addlinespace
Capability & New tool or service; network enabled; broader filesystem roots; new credential; changed tool implementation & Required or reachable protected effects may have changed even when the task name has not & Any direction & Recompute active authority and its effect bound; contract or request a scoped grant \\
\addlinespace
Delegation & New subagent or remote planner; changed child model or prompt; recursive delegation; inherited execution environment & A new subject may exercise ambient authority that the user never evaluated & Rise or redistribute & Attenuate child authority and bind it to the delegated subtask \\
\addlinespace
Task or phase & Inspect to edit; test to commit; local work to push or deploy; draft to send; browse to transact & The consequences and irreversibility of the next action have changed & Any direction & Use phase-bound active authority and explicit transition gates \\
\addlinespace
Trust context & Repository, webpage, email, tool output, or memory becomes persistent control input & Behavior may be redirected without a visible model or tool change & Often nominally unchanged & Track provenance; restrict persistence; review before high-impact effects \\
\addlinespace
Enforcement & Approval policy, sandbox, verifier, logging, commit gate, or rollback mechanism changes & The mechanism that made the original grant safe has weakened or become agent-controlled & May stay unchanged & Fail closed and require an external update path \\
\bottomrule
\end{tabularx}
\end{table*}

\paragraph{Control state and trust context.}
Self-evolution makes control-state mutation central.
Experience can be distilled into persistent rules or skills, and safety can drift through ordinary memory accumulation as well as adversarial manipulation \cite{zhao2026experience, lin2026selfevolving, shao2026misevolution, wang2026oep}.
The relevant boundary is not whether bytes changed, but whether the changed component can influence a protected action under the current grant.
Content that informs one local answer may be irrelevant; content persisted into memory, a skill, or a future plan changes the control policy that will reuse the grant \cite{chen2026dymalskill, dash2026memory}.

\paragraph{Capability and delegation.}
A manifest change is useful evidence that effect reachability may have changed, but equality of manifests is neither sufficient nor necessary for continuity \cite{zhou2026dynamiccapabilities}.
An unchanged wrapper may call a different remote implementation, while a benign version update may remain inside a tightly mediated effect set.
Delegation creates a separate principal problem: authority to create a child is not authority for the child to inherit the parent's full ceiling.
Child authority must therefore be attenuated and attributable to a parent grant.

\paragraph{Task phase and enforcement.}
Phase change is the clearest case that capability-only monitoring misses.
Moving from local testing to pushing a branch can change required effects while every component of the agent and tool inventory remains unchanged.
Enforcement mutation is different again: if ordinary task authority can weaken the monitor, approval policy, or audit path, the premises of Theorem~\ref{thm:compact-nonamp} fail directly.
Such changes require a higher-authority path outside the agent's current ceiling.

\section{Positioning and Related Work}
\label{sec:existing}

Existing approaches differ mainly in what they bind: a session, a tool call, a generated program, a skill, a capability manifest, or an autonomy stage.
These controls remain valuable, but mutation of the authorized subject or context cuts across their binding objects.
Our additional question is temporal: after the model, harness, memory, delegation graph, or task phase changes, is the old authorization still valid for the system now acting?

Table~\ref{tab:endpoints} shows how familiar policies arise by choosing different continuity rules and different treatment of the ceiling.
The two common extremes each collapse a task-dependent decision: session-wide access never revisits continuity, while per-action confirmation revisits it after every step.

\begin{table*}[t]
\centering
\footnotesize
\caption{
    Existing authorization policies as choices of continuity rule and authority boundary.
}
\label{tab:endpoints}
\begin{tabularx}{\textwidth}{p{0.22\textwidth} p{0.27\textwidth} Y}
\toprule
Policy & Continuity and boundary & Consequence \\
\midrule
Session-wide full access & Every transition is accepted and the ceiling is unrestricted & No mutation invalidates the grant; non-amplification becomes vacuous \\
\addlinespace
Per-action confirmation & A fresh decision is required for each protected action & Strong continuity through repeated user interaction, with maximum prompt cost \\
\addlinespace
Exact state-hash binding & Any byte-level state change invalidates the grant & Detects all changes but rejects many benign ones \\
\addlinespace
Tool-manifest binding & Capability changes invalidate; other mutations survive & Misses phase, control-state, delegation, and enforcement changes \\
\addlinespace
Earned-autonomy ladder & Performance evidence may raise the terminal boundary & No fixed ceiling remains unless the ladder terminates below one \\
\addlinespace
State-bound authorization & The transition envelope is task-specific and $B$ is immutable & Continuity and active authority can vary without allowing runtime evidence to move the boundary \\
\bottomrule
\end{tabularx}
\end{table*}

\paragraph{Earned autonomy.}
Practitioner and governance frameworks commonly expand autonomy after demonstrated performance \cite{aws2026principles, knowbe4govern, woodruff2026atf, reavis2026autonomy}.
Our objection is not to gradual allocation beneath a limit.
It is to allowing performance evidence to raise the limit itself.
Accuracy, task success, or an incident record produced inside the session may support a reliability decision, but a security claim below $B$ additionally needs an external premise such as an independent audit, authorized countersignature, or evaluation the agent cannot run on itself.
The ladder must also terminate beneath a boundary that no rung can move.

\paragraph{Tool and consent policies.}
Progent constrains tool calls and allows silent policy narrowing while requiring approval for expansion \cite{shi2025progent}.
ConLeash uses lattice refinement to reduce both durable over-granting and prompt fatigue \cite{li2026conleash}, while AgenticRei adds runtime obligations and dispensations enforced outside the model \cite{joshi2026deontic}.
These systems govern the permitted action space or the evolution of policy.
State-bound authorization addresses the complementary case in which the policy is unchanged but the subject or authorization context has changed.
Their policy languages can instantiate the effect ceiling $B$ and active authority $P_t$.

\paragraph{Programs, skills, and capability flow.}
LBAC constrains generated programs through types and runtime enforcement \cite{zhou2026lbac}; SkillGuard treats skills as permission-bearing artifacts \cite{pan2026skillguard}; and CaMeL tracks trusted control and data flow before tool calls \cite{debenedetti2025camel}.
These mechanisms are natural implementations of the sound request abstraction required by Equation~\ref{eq:request-sound}.
Capability binding separately invalidates certificates when tool manifests change \cite{zhou2026dynamiccapabilities}.
Our transition envelope generalizes the invalidation question to control state, memory, phase, delegation, and enforcement changes that may leave the manifest untouched.

\paragraph{Self-evolving agents.}
The self-evolution literature explains what changes and how those changes persist across sessions or descendants \cite{gao2026selfevolvingsurvey, lin2026selfevolving, zhao2026experience}.
We supply the authorization interpretation: which changes invalidate an old grant, which change the authority currently needed, and which must remain unable to move the user-issued ceiling.

\paragraph{Classical foundations.}
The design follows least privilege, complete mediation, and the reference-monitor principle \cite{saltzer1975protection, anderson1972referenceMonitor}.
Reusing a decision after subject or context mutation is a temporal relative of time-of-check/time-of-use, except that the changing object is the authorization setting rather than the protected resource.
It also resembles the confused deputy when a privileged component acts on instructions from an untrusted source \cite{hardy1988confusedDeputy}.

\section{Evaluation Path and Limitations}
\label{sec:evaluation}

The present paper formalizes authorization drift but does not yet measure its frequency or operational cost.
We therefore present it as a formalization rather than as an empirical claim about deployed-system prevalence.
A first empirical study should replay tasks across controlled mutations and record: whether an old grant persists, whether a post-mutation action escapes its original effect set, whether benign changes cause unnecessary suspension, and how often users must reauthorize.
These correspond to grant persistence rate, post-mutation unauthorized-effect rate, false invalidation rate, and interaction cost.

Three limitations are central.
First, choosing $B$ remains a policy and usability problem: a high ceiling weakens the result, while a low ceiling interrupts useful work.
Second, Equation~\ref{eq:request-sound} needs a concrete enforcement mechanism such as syscall mediation, tool proxies, type-and-effect analysis, or capability-flow tracking.
Label delivery alone is insufficient if the resulting monitor view under-approximates actual effects.
Third, cross-generation revocation and population-level authority are not solved here.

\section{Conclusion}
\label{sec:conclusion-main}

An evolving agent should not inherit a static grant merely because its session identity persists, and an unchanged agent should not carry one across an unanticipated task or enforcement transition.
Its active authority may need to change, but the change must be conditional and bounded.
State-bound authorization supplies two independent answers: a transition envelope decides whether the old grant survives, and a fixed effect ceiling bounds all authority that can become active beneath it.
This distinction turns ``earned autonomy'' into a usable allocation policy without mistaking agent-produced evidence for a security boundary.

\bibliographystyle{IEEEtran}
\bibliography{ref}

\appendix

\section{Detailed Structured Model}
\label{sec:model}

This appendix instantiates the effect-set model with structured resources, phases, budgets, mutation labels, and delegation constraints.
The correspondence with the compact model is explicit:
\begin{equation}
    \begin{aligned}
        B &:= \mathsf{Scope}(L_{\max}),
        & P_t &:= \mathsf{Scope}(L_t), \\
        P_t^{\mathrm{prop}} &:= \mathsf{Scope}(L_t^{\mathrm{prop}}).
    \end{aligned}
    \label{eq:model-projection}
\end{equation}
This projection deliberately forgets goal, phase, budget, and delegation-template fields; the structured validity predicate restores those checks separately.
The compact predicate $\mathsf{Alive}_G(t)$ corresponds to path compatibility under the grant-fixed envelope $\mathcal{M}$.
Theorem~\ref{thm:compact-nonamp} assumes request soundness and root-ceiling containment directly, whereas Theorem~\ref{thm:nonamp} derives root containment through structured lease ancestry and states closure soundness as the operational monitor-view obligation.
They establish the same ceiling-level boundary through different effect-soundness interfaces; the latter is a structured counterpart, not merely a change of notation.

For this instantiation, let $\Lambda$ be the mutation-label alphabet, $\mathcal{E}$ the protected-effect universe, and write the full security-relevant state as
\begin{equation}
    S_t = \langle m_t,h_t,k_t,c_t,d_t,q_t,e_t \rangle,
\end{equation}
covering the model, harness, persistent knowledge, capability configuration, delegation graph, task phase, and enforcement state.
The first five components describe the authorized subject, while $q_t$ and $e_t$ describe its authorization context.
We define leases and their order (Section~\ref{sec:leases}) and the effects a lease is meant to bound (Section~\ref{sec:scope}), establish that the authority a task step requires is not invariant under mutation (Section~\ref{sec:need}), fix a ceiling with evidence-gated movement beneath it (Section~\ref{sec:ceiling}), and characterize how much authority such a gate can concede (Section~\ref{sec:reach}).
We then give mutation envelopes as sets of labeled transitions (Section~\ref{sec:envelope}) and a validity predicate that separates mechanically checkable conditions from semantic ones (Section~\ref{sec:validity}), state authorization continuity, prove a non-amplification theorem, and give a construction showing that one of its assumptions cannot be dropped.

\subsection{Leases and the Authority Order}
\label{sec:leases}

\begin{definition}[Grant context and lease]
A grant fixes the context
\begin{equation}
    C_G = \langle \mu_0, \mathcal{M}, \mathcal{D} \rangle,
\end{equation}
where $\mu_0 = \mathsf{Measure}(S_0)$ is the grant-time measurement, $\mathcal{M} \subseteq \Lambda$ is the mutation envelope, and $\mathcal{D}$ contains the admissible child-lease templates.
Within one context, an authority lease is
\begin{equation}
    L = \langle C_G; g, p, R, E, \tau \rangle.
\end{equation}
Here $g$ is the goal, $p$ is the phase, $R$ is the resource set, $E \subseteq \mathcal{T}$ is the set of effect types, and $\tau \in \mathbb{N} \cup \{\infty\}$ is a protected-action budget.
\end{definition}

We fix $\tau$ as a usage budget rather than a wall-clock deadline so that a single scalar carries the lifecycle bound and the meet of Proposition~\ref{prop:semilattice} is $\min$.
A deployment wanting both can carry a pair ordered componentwise without affecting any result below.
Freshness is the corresponding predicate: $\mathsf{Fresh}(L,t)$ holds iff $L$ has not been revoked or suspended before $t$ and fewer than $\tau$ protected actions have been admitted under $L$ since it was granted.

To compare leases we need an order. Let $\sqsubseteq_g$ be a refinement preorder on goals and phases, under which $g' \sqsubseteq_g g$ means that $g'$ is a sub-goal of $g$. We assume $\sqsubseteq_g$ has binary meets; this is a modeling assumption, and in practice it is discharged by drawing goals and phases from a fixed task ontology rather than from free text.

\begin{definition}[Authority order]
For leases in the same grant context $C_G$, $L' \preceq L$ iff $g' \sqsubseteq_g g$, $p' \sqsubseteq_g p$, $R' \subseteq R$, $E' \subseteq E$, and $\tau' \leq \tau$.
Leases from different grant contexts are not compared.
\end{definition}

\begin{proposition}[Semilattice]
\label{prop:semilattice}
For each fixed context $C_G$, the leases ordered by $\preceq$ form a meet-semilattice with componentwise meet
\begin{equation}
    L_1 \sqcap L_2 = \langle C_G;
    g_1 \sqcap g_2, p_1 \sqcap p_2, R_1 \cap R_2,
    E_1 \cap E_2, \min(\tau_1,\tau_2) \rangle,
\end{equation}
\end{proposition}

The result is the product construction on the five authority-bearing fields.
The grant measurement, mutation envelope, and delegation templates are excluded from the order because they determine validity, not amount of authority.
This separation avoids treating incomparable measurements as if they had an authority meet.
Proposition~\ref{prop:semilattice} gives the delegation constraint of Section~\ref{sec:taxonomy} its meaning for a parent and child governed by the same context.

\subsection{Scope, the Authority Ceiling, and Effective Authority}
\label{sec:scope}

A lease names permitted resources and effect types; the corresponding set of effects is its scope.

\begin{definition}[Scope]
\label{def:scope}
$\mathsf{Scope}(L) = \{\epsilon \in \mathcal{E} : \mathsf{type}(\epsilon) \in E \wedge \mathsf{res}(\epsilon) \subseteq R\}$.
\end{definition}

For leases in the same grant context, Definition~\ref{def:scope} and Proposition~\ref{prop:semilattice} give
\begin{equation}
    \mathsf{Scope}(L_1 \sqcap L_2)
    = \mathsf{Scope}(L_1) \cap \mathsf{Scope}(L_2).
    \label{eq:scope-meet}
\end{equation}
Consequently, the structured clamp projects exactly to Equation~\ref{eq:clamp} under Equation~\ref{eq:model-projection}.

\begin{proposition}[Monotonicity]
\label{prop:monotone}
If $L' \preceq L$ then $\mathsf{Scope}(L') \subseteq \mathsf{Scope}(L)$.
\end{proposition}

This follows directly from $R' \subseteq R$ and $E' \subseteq E$.
Monotonicity is the property that makes attenuating delegation useful: a child's declared scope cannot exceed its parent's, while the runtime assumptions below connect that declared scope to realized effects.

Monotonicity is also what licenses the way we speak about the model for the rest of this section. Because $\preceq$ implies containment of scopes, every object we define is a region of $\mathcal{E}$ and every relation between them is set inclusion: the ceiling is an outer boundary, the active lease is a region inside it, the clamp is intersection with the ceiling, and the reachable authority of Section~\ref{sec:reach} is a union of regions.

\subsubsection{Why authority need moves}
\label{sec:need}

Before fixing a boundary we should say what it bounds, and in particular whether the authority an agent requires is stable across a session at all.
It is not.
The reason is worth separating from the security argument, because it is the premise on which the rest of the model rests.

\begin{definition}[Authority need]
\label{def:need}
Let $\mathsf{Plan}(S_t,q_t)$ be the set of action sequences that complete the current task step $q_t$ from state $S_t$.
The \emph{need} at $t$ is
\begin{equation}
    \begin{split}
        \mathsf{Need}(S_t,q_t) \;=\; \{\, X \subseteq \mathcal{E} \;:\; &\exists\, \alpha \in \mathsf{Plan}(S_t,q_t), \\
        &\forall a \in \alpha,\ \mathsf{Req}(a) \subseteq X \,\},
    \end{split}
\end{equation}
the family of effect bounds under which some completing plan is admissible.
\end{definition}

$\mathsf{Need}$ is upward closed under set inclusion, since enlarging an effect bound never removes a plan.
It need not have a least element, because two incomparable effect sets may each admit a different completing plan, which is why we define need as a family rather than as a single minimum.
The compact proposal $P_t^{\mathrm{prop}}$ may select a minimal member; a structured proposal $L_t^{\mathrm{prop}}$ realizes that choice when its scope contains the selected effect set.

\begin{proposition}[Need is not invariant under mutation]
\label{prop:needmoves}
There are transitions $S_t \xrightarrow{\lambda} S_{t+1}$ with $q_{t+1} = q_t$ for which $\mathsf{Need}(S_{t+1},q_{t+1}) \neq \mathsf{Need}(S_t,q_t)$.
Across such transitions, a minimal effect bound can rise, fall, or be replaced by an incomparable one.
There are also phase transitions in which every state component except $q$ remains unchanged and the need changes.
\end{proposition}

\begin{proof}[Construction]
For a capability transition, let $q$ be the unit-testing step.
At $S_t$ the constructed system has a single completing plan whose protected effect is local test execution, so the corresponding singleton is a minimal member of $\mathsf{Need}(S_t,q)$.
Now replace that runner with one that also reports to a remote collector, leaving no local-only completing plan.
Every completing plan then requires outbound network transmission in addition to local execution.
The upward-closed family $\mathsf{Need}$ shrinks, while its minimal effect bound rises.
Conversely, replacing a mandatory networked check with a local analyzer enlarges $\mathsf{Need}$ and lowers a minimal requirement.
For an incomparable change, let the only initial plan require reading a local credential and let the only post-mutation plan instead require invoking a remote service without that read.
The two minimal effect bounds contain different effect types, so neither contains the other.

For a phase transition, let $q$ move from editing and testing to pushing a branch, with every other component of $S$ unchanged.
No model, harness, or capability component changes, no tool is added or removed, and the capability manifest is bit-identical.
Every completing plan for the new step requires a remote reference update, so every member of $\mathsf{Need}$ must contain that effect and the family changes.
\end{proof}

Both witnesses are instantiated in a coding setting for concreteness, but neither uses anything specific to it.
The first needs only that some transition changes which effect types a completing plan requires; the second needs only a task with an irreversible stage that the tool inventory does not distinguish, which is the ordinary situation for an assistant that drafts before sending or an operator that diagnoses before remediating.

Proposition~\ref{prop:needmoves} answers the question this paper opens with, and answers it in the way that creates the difficulty rather than resolving it.
Authority need does move when the subject or authorization context changes, so a grant fixed as a single point is either too small to finish the task or too large for most of it.
But need is computed from the agent's own state and plan, and an agent whose state has been mutated by an adversary has a mutated need; a system that simply tracked need would be granting authority on the attacker's say-so.
The rest of this section is the consequence.
Authority may follow need, but only beneath a boundary that need cannot move, and the second construction shows why that boundary cannot be indexed on the capability manifest: it is exactly the case in which nothing observable in the manifest changes.

\subsubsection{The authority ceiling}
\label{sec:ceiling}

A single fixed lease is not what a user issuing a broad grant intends. The intent is better captured as: \emph{at grant time we fix a boundary and a tolerance for movement inside it; as the task proceeds and the agent's instructions, tools, skills, and ingested context all change, the authority actually in force changes with them, but it must never exceed the boundary.} We therefore let a grant issue an interval rather than a point.

\begin{definition}[Authority ceiling and gated ascent]
\label{def:band}
A grant fixes a context $C_G$, a \emph{ceiling} $L_{\max}$ in that context, an initial active lease $L_0 \preceq L_{\max}$, and an evidence relation $\vdash$.
Every active lease $L_t$ remains in $C_G$ and satisfies $L_t \preceq L_{\max}$.
The ceiling is \emph{fixed}: no runtime signal whatsoever, including evidence about the agent's competence or track record, may raise it, and raising it requires a fresh user grant.
Below the ceiling, movement is asymmetric.
Descent is free: a proposal $L_t^{\mathrm{prop}} \preceq L_{t-1}$ is admitted unconditionally.
Every other proposal counts as \emph{ascent}, including one merely incomparable to $L_{t-1}$, and is \emph{gated}: it is admitted only if $\mathsf{Evid}(t) \vdash L_t^{\mathrm{prop}}$ holds, where $\mathsf{Evid}(t)$ is the evidence derivable from the execution prefix up to $t$.
In both cases the result is clamped,
\begin{equation}
    L_t \;=\; L_t^{\mathrm{prop}} \sqcap L_{\max},
\end{equation}
which is well defined by Proposition~\ref{prop:semilattice} and requires only the meet.
No join is needed, because ascent is expressed as a proposal rather than as a combination of the current lease with an increment; a system that merged two grants would additionally require a lattice.

Float acts only on the authority components $(g,p,R,E,\tau)$.
The context $C_G$, including $\mathcal{M}$ and $\mathcal{D}$, is grant-fixed and is not an operand of the clamp.
\end{definition}

\begin{figure*}[t]
\centering
\includegraphics[width=0.94\textwidth]{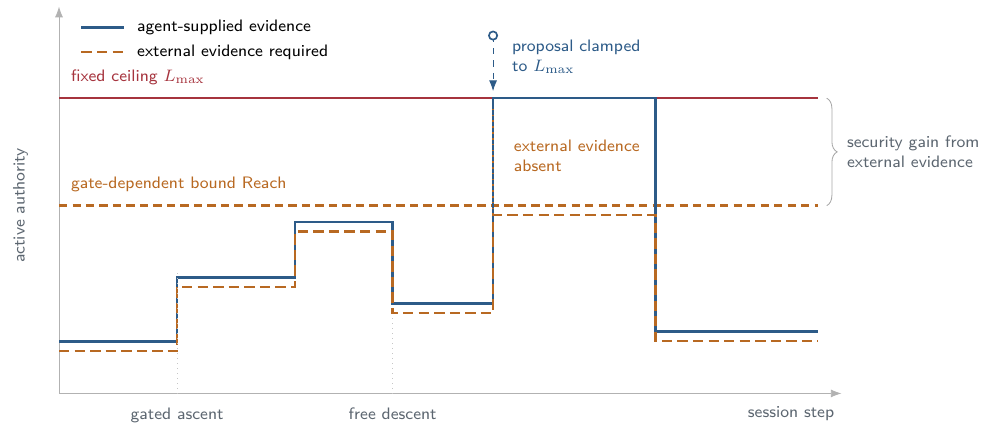}
\caption{
    Authority beneath a fixed ceiling.
    The solid line is the active lease $L_t$ under an unconditionally realizable evidence relation (Proposition~\ref{prop:bandtop}); it ascends on evidence, descends freely under decay, and is clamped when a proposal would exceed the ceiling.
    The dashed line illustrates a trajectory under a relation with a premise outside the agent's causal reach; $\mathsf{Reach}$ is a sound history-independent bound on such trajectories (Theorem~\ref{thm:reach}, Proposition~\ref{prop:tighten}).
    The gap between the horizontal bounds is what the analysis can certify for the gate, not a claim that every point is attainable.
    The envelope guard is orthogonal to all of this and suspends the grant outright.
}
\label{fig:float}
\end{figure*}

This is the paper's central structure, and it is a layering rather than a single mechanism. The ceiling is a security boundary and holds against an adversary. The evidence relation $\vdash$ governs position beneath that boundary and is a reliability mechanism: it exists to keep the agent from being handed authority it will mishandle, not to keep an attacker from acquiring authority. Section~\ref{sec:noadv} shows why this division lets $\vdash$ be far weaker than it would otherwise have to be.

\subsubsection{A gate-dependent permission bound}
\label{sec:reach}

Gated ascent is easy to misread in either direction.
Read one way, the gate is a security control that keeps authority below the ceiling; that reading fails whenever the agent can manufacture the evidence.
Read the other way, the gate is decoration; that reading fails too, and the difference between them turns on a single property of $\vdash$ that we now isolate.

Evidence does not arrive from nowhere.
$\mathsf{Evid}(t)$ is a function of the execution prefix, and the agent's contribution to that prefix is constrained by the authority it held while producing it.
An agent cannot exhibit a passing integration-test record if it may not run tests, and it cannot demonstrate safe handling of a deployment tool it has never been able to invoke.
Ascent is therefore a bootstrapping process, and what it can reach is a closure rather than the ceiling by default.

\begin{definition}[Realizability]
\label{def:realizable}
Let $\mathsf{Beh}(\mathcal{U})$ be the set of finite behaviors the agent may execute in a session whose active lease is drawn from $\mathcal{U}$ at every step, and let $\mathsf{Evid}(b)$ be the evidence induced by a behavior $b$.
Write $\mathcal{U} \rightsquigarrow L^{*}$ if some $b \in \mathsf{Beh}(\mathcal{U})$ satisfies $\mathsf{Evid}(b) \vdash L^{*}$.
\end{definition}

The existential quantifier is deliberately attacker-favorable: it suffices that one available behavior produces qualifying evidence, whether or not a benign agent would choose it.

\begin{definition}[Escalation closure]
\label{def:closure}
Write ${\downarrow}X = \{L : \exists L' \in X,\ L \preceq L'\}$ for downward closure.
For a grant $\langle L_{\max}, L_0, \vdash \rangle$, the \emph{escalation closure} $\mathcal{U}_{\infty}$ is the least set of leases satisfying
\begin{enumerate}
    \item $L_0 \in \mathcal{U}_{\infty}$ and every member is below $L_{\max}$;
    \item $\mathcal{U}_{\infty}$ is downward closed; and
    \item if $\mathcal{U}_{\infty} \rightsquigarrow L^{*}$, then $L^{*} \sqcap L_{\max} \in \mathcal{U}_{\infty}$.
\end{enumerate}
Its permission bound is
\begin{equation}
    \mathsf{Reach}(L_{\max}, L_0, \vdash) \;=\; \bigcup_{L \in \mathcal{U}_{\infty}} \mathsf{Scope}(L).
\end{equation}
\end{definition}

The full set $\{L:L \preceq L_{\max}\}$ satisfies the three conditions.
The intersection of all sets satisfying them also does so because realizability is monotone under set inclusion, hence a least such set exists.
The construction intentionally forgets which witnesses occur in mutually compatible histories.
It is therefore a history-independent over-approximation of the permissions that gated ascent may make active, not an exact set of effects that executions can realize.

\begin{theorem}[Gate-dependent upper bound]
\label{thm:reach}
Under the assumptions of Theorem~\ref{thm:nonamp}, every execution $\pi$ under a single grant $\langle L_{\max}, L_0, \vdash \rangle$ satisfies
\begin{equation}
    \mathsf{Effects}(\pi) \;\subseteq\; \mathsf{Reach}(L_{\max}, L_0, \vdash),
\end{equation}
without claiming that every member of the right-hand side is realizable.
\end{theorem}

\begin{proof}
For containment we show by induction on $t$ that the active lease satisfies $L_t \in \mathcal{U}_{\infty}$.
The base case and descent case follow from the first two clauses of Definition~\ref{def:closure}.
If step $t$ is a gated ascent, the gate fired on $\mathsf{Evid}(t)$, which is induced by the execution prefix $b$ up to $t$.
By the induction hypothesis every active lease along that prefix lies in $\mathcal{U}_{\infty}$, so $b \in \mathsf{Beh}(\mathcal{U}_{\infty})$ and hence $\mathcal{U}_{\infty} \rightsquigarrow L_t^{\mathrm{prop}}$.
The third clause then gives $L_t^{\mathrm{prop}} \sqcap L_{\max} = L_t \in \mathcal{U}_{\infty}$.
By Theorem~\ref{thm:nonamp} the effects of the action admitted at step $t$ lie in $\mathsf{Scope}(L_t)$, and taking the union over $t$ gives the claim.
\end{proof}

\begin{corollary}[The ceiling is a sound over-approximation]
\label{cor:ceilingover}
$\mathsf{Reach}(L_{\max}, L_0, \vdash) \subseteq \mathsf{Scope}(L_{\max})$ for every $L_0$ and every $\vdash$.
\end{corollary}

\begin{proof}
Every $L \in \mathcal{U}_{\infty}$ satisfies $L \preceq L_{\max}$ by the first clause of Definition~\ref{def:closure}.
Monotonicity (Proposition~\ref{prop:monotone}) then gives $\mathsf{Scope}(L) \subseteq \mathsf{Scope}(L_{\max})$.
Note that this argument never mentions $\vdash$.
\end{proof}

Theorem~\ref{thm:nonamp} discards the gate-dependent closure and retains only the deployable ceiling.
The gap between the bounds is a sound estimate of what a gate may exclude; it need not be an attainable effect region.
We now characterize both ends of that gap.

\begin{proposition}[When the ceiling is tight]
\label{prop:bandtop}
Call $\vdash$ \emph{unconditionally realizable} if $\mathcal{U}_{\infty} \rightsquigarrow L^{*}$ for every $L^{*} \preceq L_{\max}$.
For such $\vdash$ we have $\mathsf{Reach}(L_{\max},L_0,\vdash) = \mathsf{Scope}(L_{\max})$, so this analysis cannot certify a bound below the ceiling.
\end{proposition}

\begin{proof}
Containment is Corollary~\ref{cor:ceilingover}.
Conversely, the hypothesis places every $L^{*} \preceq L_{\max}$ in $\mathcal{U}_{\infty}$, in particular $L_{\max}$ itself, so $\mathsf{Scope}(L_{\max}) \subseteq \mathsf{Reach}$.
\end{proof}

\begin{remark}[Self-realizable evidence]
\label{rem:commoncase}
Proposition~\ref{prop:bandtop} is conditional.
Its hypothesis is plausible for gates based only on evidence that the session can itself realize, such as some task outcomes, test records, declared effect manifests, or execution history.
Whether the hypothesis holds for a concrete gate must be established from that gate's provenance and the session's causal capabilities; it does not follow from the phrase ``behavior-derived'' alone.
\end{remark}

\begin{proposition}[When a gate strictly tightens the bound]
\label{prop:tighten}
Suppose $L^{*} \preceq L_{\max}$ but $L^{*} \npreceq L_0$, and every derivation of $\mathsf{Evid} \vdash L'$ whose clamped result satisfies $L^{*} \preceq L' \sqcap L_{\max}$ requires a premise $\nu$ outside the agent's causal reach.
Suppose also that $\nu$ does not hold independently during the session.
Then $L^{*} \notin \mathcal{U}_{\infty}$, and whenever $\mathsf{Scope}(L^{*}) \not\subseteq \mathsf{Reach}(L_{\max},L_0,\vdash)$ the inclusion of Corollary~\ref{cor:ceilingover} is strict.
\end{proposition}

\begin{proof}
Let $\mathcal{V}=\{L:L\preceq L_{\max}\ \wedge\ L^{*}\npreceq L\}$.
The assumption $L^{*}\npreceq L_0$ places $L_0$ in $\mathcal{V}$, and $\mathcal{V}$ is downward closed.
No behavior using leases in $\mathcal{V}$ can establish $\nu$, so Definition~\ref{def:realizable} and the hypothesis rule out every proposal whose clamped result dominates $L^{*}$.
Thus $\mathcal{V}$ satisfies all three closure conditions in Definition~\ref{def:closure}.
By leastness, $\mathcal{U}_{\infty}\subseteq\mathcal{V}$, and therefore $L^{*}\notin\mathcal{U}_{\infty}$.
\end{proof}

\begin{remark}[The design rule]
\label{rem:designrule}
Propositions~\ref{prop:bandtop} and~\ref{prop:tighten} give an implementer a conservative rule: \emph{a gate can support a security bound below the ceiling only through premises that the session cannot itself realize.}
A gate whose qualifying evidence is unconditionally realizable is a reliability mechanism and should not be described as a tighter adversarial bound.
Evidence provenance alone does not decide the issue; the relevant property is whether the session can realize every ascent path under the authority available along that path.
A gate whose premises include an external attestation, a human countersignature, a wall-clock or calendar condition, or an evaluation the agent cannot run on itself can keep selected high-impact authority outside $\mathcal{U}_{\infty}$ while that premise is absent.
The recipe follows directly: make every path to the most dangerous authority depend on a premise the session cannot manufacture, and the escalation closure can stop strictly below the ceiling.
\end{remark}

\begin{remark}[$\mathcal{U}_{\infty}$ is a specification, not an algorithm]
\label{rem:reachundecidable}
$\mathcal{U}_{\infty}$ is no more computable than $\mathcal{A}$, since deciding $\mathcal{U} \rightsquigarrow L^{*}$ quantifies over behaviors that include shell execution.
The sound direction is the same as in Remark~\ref{rem:undecidable}, and $\mathsf{Scope}(L_{\max})$ is a computable over-approximation of $\mathsf{Reach}$ that is available at grant time with no analysis at all.
This is why we state the deployable guarantee as Theorem~\ref{thm:nonamp} in terms of the ceiling, and treat Theorem~\ref{thm:reach} as the characterization that says how much a particular gate design is worth.
\end{remark}

\begin{remark}[What the ceiling buys, and what gating buys]
\label{rem:bandvalue}
Under Proposition~\ref{prop:bandtop}, authorizing the ceiling is adversarially equivalent to authorizing everything beneath it, so momentary $L_t$ is not a tighter security bound.
Gating can still reduce expected exposure and non-adversarial blast radius, lower prompt cost, and provide an auditable authority trajectory.
Ceiling occupancy is consequently a utility measure; a security claim below the ceiling requires the condition of Proposition~\ref{prop:tighten}.
\end{remark}

\subsubsection{Why the evidence rule need not be adversarially robust}
\label{sec:noadv}

Proposition~\ref{prop:bandtop} reads at first like a negative result about earned authority, and much practitioner guidance, which holds that autonomy should be expanded progressively on demonstrated performance rather than granted by default, is vulnerable to exactly that reading.
If competence evidence can be manufactured by the entity that benefits from it, then conditioning authority on competence appears to be a privilege-escalation channel by construction.

The layering of Definition~\ref{def:band} dissolves this objection rather than answering it.
Because the ceiling is fixed and Corollary~\ref{cor:ceilingover} holds independently of $\vdash$, the worst outcome of a fully compromised evidence channel is that $L_t$ rises to $L_{\max}$, which is the bound the system already guarantees.
The failure of the evidence rule is \emph{absorbed} by the boundary above it.
Consequently $\vdash$ carries no additional obligation for the \emph{ceiling-level non-amplification claim}.
It does carry a security obligation if the system advertises a stronger, gate-dependent bound below the ceiling.
For reliability-only allocation it may use signals that would be unacceptable as security boundaries, including test pass rates or historical error rates.

This is what makes competence gating implementable.
A competence signal is worthless as a security control, since demonstrating fluency in a language says nothing about resistance to injection, and a package manager that runs installation scripts confers arbitrary execution regardless of how well the agent writes code.
It is nonetheless perfectly serviceable as a reliability control situated beneath a security boundary that does not depend on it.
The two questions must be answered by different mechanisms: \emph{will this agent mishandle the authority} is answered by evidence, and \emph{what is the most this agent could do if that evidence is wrong} is answered by the ceiling.
Conflating them produces either a security mechanism resting on unreliable signals or a reliability mechanism burdened with adversarial requirements it cannot meet.

We state the resulting design obligation as a property of a system rather than as advice, because it is the paper's central normative claim and we want it to be checkable.

\begin{definition}[Separation obligation]
\label{def:separation}
An authorization system satisfies the \emph{separation obligation} if every grant fixes a bound $L_{\max}$ at grant time such that no signal admissible as an input to $\vdash$ can raise it.
\end{definition}

\begin{proposition}[The obligation is what makes the bound evidence-independent]
\label{prop:separation}
Let a system satisfy Definition~\ref{def:separation}.
Then $\mathsf{Effects}(\pi) \subseteq \mathsf{Scope}(L_{\max})$ for every evidence relation $\vdash$, including one chosen adversarially and one that is entirely compromised.
Let a system instead violate it, so that evidence advances a grant along a chain of bounds $L^{(1)} \preceq \cdots \preceq L^{(K)}$.
Then its guaranteed bound is $\mathsf{Scope}(L^{(K)})$ whenever the stage relation is unconditionally realizable in the sense of Proposition~\ref{prop:bandtop}.
\end{proposition}

\begin{proof}
The first claim is Corollary~\ref{cor:ceilingover}, whose proof uses only the clamp and Proposition~\ref{prop:monotone} and never refers to $\vdash$.
For the second, read the stage ladder as a single grant with ceiling $L^{(K)}$ and with $\vdash$ the stage-advancement relation, and apply Proposition~\ref{prop:bandtop}.
\end{proof}

Maturity models in which demonstrated performance removes the last ceiling at an unsupervised stage violate Definition~\ref{def:separation} by construction.
Proposition~\ref{prop:separation} prices that choice: the guarantee such a system offers is not the stage the agent currently occupies but the highest stage its own conduct can reach.

\begin{remark}[Authority ratchets upward without decay]
\label{rem:decay}
Descent is free but nothing in Definition~\ref{def:band} compels it, and evidence accumulates monotonically over a session: once $\mathsf{Evid} \vdash L^{*}$ has been established it does not spontaneously cease to hold.
The path of least resistance is therefore for the runtime to keep proposing whatever it last justified, so $L_t$ drifts upward to $L_{\max}$ and stays there, at which point the ceiling delivers none of the expected-case benefits of Remark~\ref{rem:bandvalue}.
A practical instantiation therefore needs an explicit \emph{decay} rule, for example dropping any component of $L_t$ not exercised within $k$ steps or since the last phase boundary.
Decay is free in the sense that it needs no user event, and it is what makes elasticity worth implementing at all.
It does not, however, shrink the escalation closure: realizability is an existential condition over behaviors, and a lease that decays can be re-earned by the behavior that earned it before.
Decay improves ceiling occupancy; it does not improve the bound of Theorem~\ref{thm:reach}.
\end{remark}

\subsubsection{Effective authority}
\label{sec:effauth}

Scope is a static notion. What we actually need to bound is everything an agent could achieve by chaining permitted actions, since a file-read tool and an outbound-network tool may each appear low risk while jointly enabling exfiltration.

\begin{definition}[Effective-authority closure]
\begin{equation}
    \mathcal{A}(S,L) = \left\{\epsilon \in \mathcal{E} \;\middle|\;
    \begin{array}{l}
        \epsilon \text{ is produced by a finite}\\
        \text{sequence of actions, each}\\
        \text{permitted under } L, \text{ from } S
    \end{array}
    \right\}.
\end{equation}
\end{definition}

\begin{remark}[$\mathcal{A}$ is a specification, not an algorithm]
\label{rem:undecidable}
Computing $\mathcal{A}$ exactly is not possible in general: permitted actions include shell execution, so reachability of an effect subsumes the halting problem, and external services are opaque. A deployed system must therefore use an over-approximation $\mathcal{A}^{\#}$ satisfying $\mathcal{A}(S,L) \subseteq \mathcal{A}^{\#}(S,L)$. Over-approximation is the sound direction: the check $\mathcal{A}^{\#}(S,L) \subseteq \mathsf{Scope}(L)$ implies $\mathcal{A}(S,L) \subseteq \mathsf{Scope}(L)$, at the cost of false invalidations.
\end{remark}

\subsection{Mutation Envelopes over Paths}
\label{sec:envelope}

The natural formulation of ``these two agent states may share authority'' is a relation $S_i \equiv_L S_j$. We deliberately do not adopt it as the operational criterion, for a reason worth stating explicitly. What a runtime can observe is not an unordered pair of states but a sequence of transitions, and the security-relevant question is whether \emph{each step taken} was one the grant anticipated. Formulating the criterion over paths also exposes a failure mode that a state-pair formulation obscures, namely authority that accumulates across steps each of which is individually permitted.

\begin{definition}[Path compatibility]
\label{def:pathcompat}
Let $\sigma = \lambda_1 \lambda_2 \cdots \lambda_t$ be the sequence of mutation labels observed since the grant was issued, and let $\mathcal{M}$ be its grant-fixed envelope (Definition~\ref{def:band}).
Then $\mathsf{PathCompat}(L,\sigma)$ holds iff $\lambda_i \in \mathcal{M}$ for all $i \leq t$, equivalently $\sigma \in \mathcal{M}^{*}$.
\end{definition}

Path compatibility is closed under composition by construction, so it does not suffer the transitivity defect that a naive state-equivalence relation would. But it is a purely syntactic condition on labels, and syntactic permission does not by itself bound authority.

\begin{remark}[Why the envelope must be grant-fixed]
\label{rem:envfixed}
$\mathsf{PathCompat}$ quantifies over the entire history since the grant, so an envelope that floated with $L_t$ would be incoherent.
Under a floating envelope, the free descent of Definition~\ref{def:band} or the decay of Remark~\ref{rem:decay} could shrink $\mathcal{M}$ and thereby retroactively invalidate labels that were admissible when they occurred, so that a grant would fail for having voluntarily given up authority.
Exempting $\mathcal{M}$ and $\mathcal{D}$ from float removes this, and it is also what makes the independence claimed in Remark~\ref{rem:twoguards} literally true: were $\mathcal{M}$ elastic it would be constrained by $\preceq$, and the two guards would be coupled through the clamp.
Delegation does not use the authority order to modify these fields.
A child either remains in the parent's context and is checked against $\mathcal{D}$, or receives a fresh grant context through a higher-authority issuance path.
\end{remark}

\begin{remark}[Two envelopes, two guards]
\label{rem:twoguards}
The transition envelope $\mathcal{M}$ of Definition~\ref{def:sound-envelope} and the authority ceiling of Definition~\ref{def:band} are independent, and conflating them is a modeling error we want to forestall. $\mathcal{M}$ is indexed by transition labels and asks whether the \emph{grant survives} a change; the ceiling is indexed by authority and asks how much of it may be \emph{in force}. Neither implies the other: a transition inside $\mathcal{M}$ can push $L_t^{\mathrm{prop}}$ above $L_{\max}$, and a transition outside $\mathcal{M}$ can leave the authority in force untouched. They therefore act as two guards with different responses, and Figure~\ref{fig:guards} shows that each cell of the product is occupied. A ceiling violation is handled mechanically, by the clamp, with no user event. An envelope violation suspends the grant itself and requires intervention, \emph{even when every effect in force remains beneath the ceiling}, which is precisely the paper's thesis that authorization is not reducible to effect containment. A provider switch may leave $\mathcal{A}^{\#}$ unchanged and still not be the principal the user approved.
\end{remark}

\begin{definition}[Sound envelope]
\label{def:sound-envelope}
Let $\mathcal{R}(S_0,\mathcal{M})$ be the set of states reachable from $S_0$ by transitions labeled in $\mathcal{M}$. The envelope $\mathcal{M}$ is \emph{sound} for $L$ at $S_0$ if
\begin{equation}
    \forall S \in \mathcal{R}(S_0,\mathcal{M}):\quad \mathcal{A}(S,L) \subseteq \mathsf{Scope}(L).
\end{equation}
\end{definition}

\begin{remark}[Compositional drift]
\label{rem:drift}
Envelope soundness is not automatic. An envelope may permit a sequence of labels each of which is individually innocuous while the composite state enlarges the effective-authority closure beyond $\mathsf{Scope}(L)$, for instance a permitted dependency addition followed by a permitted configuration edit that together open a network path. We call this \emph{compositional drift}. It is the principal open problem left by this model. A conservative runtime therefore does not rely on syntactic envelopes alone; at each admission point it checks $\mathcal{A}^{\#}(\hat{S}_t,L) \subseteq \mathsf{Scope}(L)$ and requires the cross-view closure soundness of (A2), while treating the syntactic envelope as a cheap pre-filter. Establishing that a given envelope is sound, so that the runtime check can be discharged statically, is the form that a soundness proof for a concrete policy should take.
\end{remark}

\subsection{Validity: Mechanical Conditions and Oracles}
\label{sec:validity}

The conditions under which an action is authorized are not epistemically uniform. Some are decidable set-membership tests; others require judging whether an action still serves the user's goal, which no runtime can settle without a semantic model. Conflating them in one conjunction hides where the difficulty lives, so we stratify.

\begin{definition}[Monitor-visible validity]
\label{def:validity}
Fix a semantic oracle $\Omega$ that decides goal and phase questions.
For a lease $L$, monitor view $\hat{S}_t$, observed label sequence $\sigma$, and action $a_t$, the predicate $\mathsf{Valid}_\Omega(L,\hat{S}_t,\sigma,a_t)$ is the conjunction of three groups:
\begin{equation}
    \begin{aligned}
        \mathrm{(i)}\;\; & \mathsf{Fresh}(L,t) \;\wedge\; \mathsf{Req}(a_t) \subseteq \mathsf{Scope}(L); \\
        \mathrm{(ii)}\;\; & \mathsf{PathCompat}(L,\sigma) \;\wedge\; \mathsf{DelegCompat}(d_t,\mathcal{D}) \\
        & \quad {}\wedge\; \mathcal{A}^{\#}(\hat{S}_t,L) \subseteq \mathsf{Scope}(L); \\
        \mathrm{(iii)}\;\; & \Omega_{\mathrm{goal}}(a_t,g) \;\wedge\; \Omega_{\mathrm{phase}}(q_t,p).
    \end{aligned}
\end{equation}
Group (i) is mechanical and decidable from the action alone.
Its second conjunct checks the request presented at the monitor boundary, not the effects later realized by execution.
Group (ii) is mechanical given $\hat{S}_t$; $d_t$ and $q_t$ here denote the authenticated delegation and phase values in that view.
Group (iii) is oracle-dependent, and is where the difficulty lives.
\end{definition}

\begin{assumption}[Sound oracle]
\label{asm:oracle}
$\Omega$ is sound: it rejects every action that falls outside the goal or phase the user authorized. It need not be complete; a conservative oracle that rejects too much costs utility, not safety.
\end{assumption}

Claims that admitted actions remain within the user's goal or phase are conditional on Assumption~\ref{asm:oracle}; the effect-containment theorem below is not.
This distinction is a real limitation and we prefer to make it visible rather than absorb it into notation.
Two consequences follow.
First, any instrument that tests this model empirically should avoid the oracle rather than approximate it, by fixing goals and phases in advance so that a benchmark policy rather than a language model's judgment determines the ground truth.
Second, constructing a usable $\Omega$, or restricting the setting so that one is unnecessary, is on the research agenda rather than solved here (Section~\ref{sec:agenda}).

\subsection{Authorization Continuity}
\label{sec:continuity}

\begin{definition}[Authorization continuity]
An execution satisfies \emph{authorization continuity} if every protected action is backed by a currently valid lease:
\begin{equation}
    \forall t, a_t:\quad \mathsf{Execute}(a_t) \Rightarrow \exists L \in \mathcal{L}_t:\ \mathsf{Valid}_\Omega(L,\hat{S}_t,\sigma_t,a_t),
\end{equation}
where $\mathcal{L}_t$ is the set of live leases and $\sigma_t$ the label sequence since each lease was granted.
\end{definition}

A \emph{continuity violation} occurs when the runtime executes an action without this monitor-visible predicate, or when the view is not closure-sound as required below.

The existential quantifier requires care. If an action may draw on any live lease, two narrow leases can jointly license behavior that neither was meant to permit, the classic confused-deputy shape. We therefore require attribution.

\begin{assumption}[Single-lease attribution]
\label{asm:attribution}
Each admitted action is attributed to exactly one live lease, and its full effect set must be valid under that lease alone. Authority from distinct leases is never combined to justify a single action.
\end{assumption}

\subsection{A Non-Amplification Theorem}
\label{sec:theorem}

We can now state what the model buys.
Let $\mathsf{Granted}(\pi)$ be the set of \emph{root} leases, that is, the ceilings issued directly by the user during execution $\pi$, and let $\mathsf{Effects}(\pi)$ be the set of protected effects produced during $\pi$.
Every live lease is related to a root in one of two ways: by \emph{float}, since the clamp of Definition~\ref{def:band} holds an active lease beneath the ceiling of its own grant, or by \emph{delegation}, since (A4) attenuates a child below its parent.
Both relations are instances of $\preceq$, which is what lets a single induction cover them.

\begin{assumption}[Runtime conditions]
\label{asm:runtime}
\leavevmode
\begin{enumerate}[label=(A\arabic*),ref=A\arabic*,leftmargin=2.75em,labelsep=0.45em,align=left,topsep=4pt,itemsep=4pt]
    \item \textbf{Complete mediation.}
    Every protected effect is produced only by an action the monitor admits, and the realized effects of that action are $\mathsf{Eff}(S_t,a_t)$.
    \item \textbf{Closure-sound monitor view.}
    Before each admission, the monitor incorporates all relevant mutation events and its resulting view satisfies $\mathcal{A}(S_t,L) \subseteq \mathcal{A}^{\#}(\hat{S}_t,L)$ for every live lease $L$.
    Event delivery is necessary for this implementation strategy but is not by itself sufficient.
    \item \textbf{Sound admission.}
    The monitor admits $a_t$ under $L$ only if $\mathsf{Valid}_\Omega(L,\hat{S}_t,\sigma_t,a_t)$ holds.
    All mechanical checks are evaluated on the authenticated monitor view; their relation to the true current state is exactly the separate obligation in (A2).
    \item \textbf{Attenuating delegation.}
    Within a grant context, every child lease satisfies $L_{\mathrm{child}} \preceq L_{\mathrm{parent}} \sqcap L_{\mathrm{subtask}}$, and delegation depth is finite.
    A child with a different context requires a fresh root grant.
    \item \textbf{Monitor integrity.}
    No transition modifies the monitor, its lease store, or its log except through a higher-authority update path outside the agent's control.
\end{enumerate}
\end{assumption}

\begin{remark}[Why (A2) is stated on the closure and not on the state]
\label{rem:a2form}
It is tempting to require $\hat{S}_t = S_t$, but labels name classes of transitions rather than byte-level diffs (Section~\ref{sec:executions}), so no discipline of label delivery can make the monitor's view equal the true state.
The guarantee needs something weaker and attainable: the view must not \emph{under}-approximate what the agent can do, so that a check passed on $\hat{S}_t$ is also passed on $S_t$.
Closure-soundness is exactly that condition, and it composes with the over-approximation of Remark~\ref{rem:undecidable}, both erring in the safe direction.
The construction of Proposition~\ref{prop:necessity} violates it in the sharpest available way: the view retains a capability configuration with no network tool while the true configuration has one, so an effect available at $S_t$ is missing from $\mathcal{A}^{\#}(\hat{S}_t,L)$ altogether.
\end{remark}

\begin{theorem}[Non-amplification]
\label{thm:nonamp}
Under Assumptions~\ref{asm:attribution} and~\ref{asm:runtime}, every execution $\pi$ satisfies
\begin{equation}
    \mathsf{Effects}(\pi) \;\subseteq\; \bigcup_{L \in \mathsf{Granted}(\pi)} \mathsf{Scope}(L).
\end{equation}
That is, mutation cannot enlarge the set of achievable effects beyond the union of the scopes the user explicitly granted, however many state transitions occur.
\end{theorem}

The proof is by induction on admitted actions and appears in Appendix~\ref{app:proof}.
Its structure is worth previewing, since it shows which assumption does which work: (A1) confines effects to admitted actions; (A2) guarantees that the closure the monitor evaluated on its view still covers the true current state rather than a stale one; (A3) then bounds the action's effects by $\mathsf{Scope}(L)$ for the attributing lease $L$; and the clamp together with (A4) and Proposition~\ref{prop:monotone} lifts that bound from a floating or delegated lease to its root.

\begin{remark}[What the theorem does not say]
\label{rem:notnoninterference}
Theorem~\ref{thm:nonamp} bounds the set of \emph{atomic effects}, not information flows. Reading a secret under one lease and sending an unrelated message under another produces two effects, each inside its lease's scope, while the composition may still disclose the secret. Effect containment is therefore strictly weaker than non-interference, and a system aiming at confidentiality needs information-flow labels in addition to the machinery here.
\end{remark}

\subsection{Necessity of a Closure-Sound View}
\label{sec:necessity}

Assumption (A2) is semantic, not bookkeeping.
Dropping closure soundness makes the bound fail by an unbounded margin; missing a mutation event is one concrete way to cause that failure.

\begin{proposition}[Closure soundness is necessary]
\label{prop:necessity}
There is a family of executions satisfying Assumption~\ref{asm:attribution} and conditions (A1), (A3), (A4), and (A5), but not (A2), for which
\begin{equation}
    \mathsf{Effects}(\pi)
    \not\subseteq \bigcup_{L \in \mathsf{Granted}(\pi)} \mathsf{Scope}(L),
\end{equation}
and in which the number of out-of-scope effects grows without bound.
\end{proposition}

\begin{proof}[Construction]
Let $L$ grant read access to a repository and write access to its source subtree, with local file updates and test execution as its permitted effect types.
Let its envelope admit application editing and plan revision.
At $S_0$ the capability configuration $c_0$ contains no network-capable tool, so $\mathcal{A}^{\#}(S_0,L) \subseteq \mathsf{Scope}(L)$ and the lease is admissible.

Now consider a tool-addition transition that installs a network-capable service.
The label is not in $\mathcal{M}$, so a monitor that observed it would refuse further admissions under $L$.
But (A2) fails: the harness hot-reloads its configuration without notifying the monitor, so $\hat{S}$ remains $S_0$ and the view is not closure-sound, because $\mathcal{A}(S_1,L)$ now contains a network effect that $\mathcal{A}^{\#}(\hat{S}_0,L)$ does not.
The agent then invokes the test runner, an action whose effect type lies in $E$, configured so that its output is routed to the new server.
Its nominal request $\mathsf{Req}(a)$ contains only the local test effect.
The monitor evaluates $\mathcal{A}^{\#}(\hat{S}_0,L)$, finds no network effect, and admits.
The realized set $\mathsf{Eff}(S_1,a)$ contains an outbound network transmission, which is not in $E$, hence outside $\mathsf{Scope}(L)$.

Note that (A3) is not violated: the monitor applied its admission rule correctly, to a view that no longer covers the state.
Repeating the invocation $k$ times yields $k$ out-of-scope effects, and since no admission ever consults the true $c_t$, $k$ is unbounded.
\end{proof}

Proposition~\ref{prop:necessity} is the formal content of the paper's motivating question. The grant remains mechanically valid, the monitor behaves correctly with respect to what it knows, and the nominal request stays inside the lease, yet the realized effect escapes. It also identifies what to measure first: not whether products enforce policies, but whether their enforcement path observes mutations at all.

\section{Expanded Limitations and Open Problems}
\label{sec:agenda}

Section~\ref{sec:evaluation} summarizes the three limitations most directly affecting the compact model.
This section expands those limitations and records further open problems exposed by the structured model.

\paragraph{Intent and ceiling placement.}
Theorem~\ref{thm:nonamp} bounds effects, not task correctness, benevolent behavior, or information flow.
Every effect guarantee is also relative to $L_{\max}$: a ceiling drawn too high makes the separation obligation weak, while one drawn too low interrupts useful work.
We give no general method for selecting it and defer that policy and usability problem to product presets, organization policy, or task protocol.

\paragraph{Evidence and external premises.}
The layering shows that weak evidence cannot move the ceiling, but it does not show that any particular signal improves expected exposure or reliability.
That question requires future measurement of which signals predict safe handling at which scope and with what decay.
Proposition~\ref{prop:tighten} additionally requires every path to selected high-impact authority to depend on a premise the session cannot manufacture.
Attestation, an independent countersignature, or a rate condition are candidates, but their interaction cost and actual independence from the agent remain open.

\paragraph{Monitorability of mutation.}
Bitwise identity is too strict and session identity too weak, leaving sound envelopes and a computable $\mathcal{A}^{\#}$ as central open problems (Remarks~\ref{rem:undecidable} and~\ref{rem:drift}).
Provider-side model changes may be hidden, and content becomes authorization-relevant only when it enters persistent control state rather than merely being read.
Effect systems, signed component classes, provider attestations, trust tiers, and provenance-aware persistence are plausible abstractions, but none yet provides a general solution.

\paragraph{Lineage and populations.}
Self-evolving agents create descendants whose behavior rests on accumulated experience, modified skills, or changed code \cite{gao2026selfevolvingsurvey}.
Across generations, continuity requires revocation propagation and a rule for when a descendant becomes a new principal; multi-agent populations can further disperse escalation across individually bounded principals \cite{schroeder2026multiagent}.
Interoperable deployment will require reconciling these definitions with external notions of agent identity \cite{nist2026agentstandards, nccoe2026agentidentity}.

\paragraph{Verification and the boundary of the guarantee.}
A weaker local verifier cannot reproduce a stronger planner's reasoning, making typed plans, effect manifests, and proof-carrying actions a promising route to both effect checking and a restricted oracle $\Omega$.
The effect bound still assumes single-lease attribution and the five runtime conditions of Assumption~\ref{asm:runtime}; goal-respecting authorization additionally requires Assumption~\ref{asm:oracle}.
An external monitor reduces the mutable trusted computing base but does not eliminate bugs, side channels, incomplete mediation, or contextual disagreement about consent.
The mechanism should enforce an explicit policy rather than attempt to predict those preferences.

\section{Proof of Theorem~\ref{thm:nonamp}}
\label{app:proof}

We prove that under Assumptions~\ref{asm:attribution} and~\ref{asm:runtime},
\begin{equation}
    \mathsf{Effects}(\pi) \subseteq \bigcup_{L \in \mathsf{Granted}(\pi)} \mathsf{Scope}(L).
\end{equation}

\begin{proof}
We first record a fact about ancestry chains.
Let $L$ be any live lease.
If $L$ is a root lease, that is a ceiling issued by the user, then $L \in \mathsf{Granted}(\pi)$.
If $L$ is an active lease floating under a grant with ceiling $L_{\max}$, then $L = L^{\mathrm{prop}} \sqcap L_{\max} \preceq L_{\max}$ by the clamp of Definition~\ref{def:band}, and $L_{\max}$ is a root lease.
Otherwise $L$ was created by delegation from some parent $L'$, and by (A4), $L \preceq L' \sqcap L_{\mathrm{subtask}} \preceq L'$.
Since delegation depth is finite and every step of the chain is an instance of $\preceq$, iterating yields a root lease $L_r \in \mathsf{Granted}(\pi)$ with $L \preceq L_r$, and by Proposition~\ref{prop:monotone}, $\mathsf{Scope}(L) \subseteq \mathsf{Scope}(L_r)$.
Call this the \emph{chain bound}.

We now induct on the number $k$ of protected actions admitted during $\pi$.

\emph{Base case} $k=0$. No protected action is admitted, so by (A1) no protected effect is produced and $\mathsf{Effects}(\pi) = \emptyset$, which is contained in any union.

\emph{Inductive step.} Suppose the claim holds after $k$ admitted actions, and consider the $(k{+}1)$-st, an action $a$ admitted at state $S_t$ with observed label sequence $\sigma_t$.
By (A1) the effects added at this step are $\mathsf{Eff}(S_t,a)$, so it suffices to show $\mathsf{Eff}(S_t,a) \subseteq \mathsf{Scope}(L_r)$ for some root lease $L_r$.

By Assumption~\ref{asm:attribution} the action is attributed to a unique live lease $L$, and its whole effect set must be valid under $L$ alone.
By (A5) the monitor's decision procedure and lease store are the intended ones, so its admission of $a$ reflects an evaluation of $\mathsf{Valid}_\Omega(L,\hat{S}_t,\sigma_t,a)$ on its view $\hat{S}_t$.
By (A3) that evaluation succeeded, so unfolding Definition~\ref{def:validity} gives in particular
\begin{equation}
    \mathcal{A}^{\#}(\hat{S}_t,L) \subseteq \mathsf{Scope}(L).
\end{equation}
By (A2), the monitor view at this admission is closure-sound:
\begin{equation}
    \mathcal{A}(S_t,L) \subseteq \mathcal{A}^{\#}(\hat{S}_t,L).
\end{equation}
Since $a$ is an action permitted under $L$ at $S_t$, the one-step sequence consisting of $a$ witnesses $\mathsf{Eff}(S_t,a) \subseteq \mathcal{A}(S_t,L)$.
Chaining the three inclusions gives $\mathsf{Eff}(S_t,a) \subseteq \mathsf{Scope}(L)$.

By the chain bound there is a root lease $L_r \in \mathsf{Granted}(\pi)$ with $\mathsf{Scope}(L) \subseteq \mathsf{Scope}(L_r)$, so $\mathsf{Eff}(S_t,a) \subseteq \mathsf{Scope}(L_r)$.
Combining with the induction hypothesis, the effects produced by the first $k{+}1$ admitted actions lie in $\bigcup_{L \in \mathsf{Granted}(\pi)} \mathsf{Scope}(L)$.
\end{proof}

Two observations about where the assumptions are used.
Soundness of $\Omega$ (Assumption~\ref{asm:oracle}) does not appear in the containment argument and is therefore not a premise of the theorem.
It is required only for the additional interpretation that admitted actions serve the authorized goal; effect containment holds even with an unsound oracle because $R$ and $E$ are checked mechanically.
Conversely (A2) is load-bearing in a single step, the passage from the closure evaluated on $\hat{S}_t$ to the closure at $S_t$, and Proposition~\ref{prop:necessity} shows that step cannot be repaired by strengthening any other assumption.

\end{document}